\documentclass[10pt,twocolumn,letterpaper]{article}

\usepackage{cvpr}
\usepackage{times}
\usepackage{epsfig}
\usepackage{graphicx}
\usepackage{amsmath}
\usepackage{amssymb}

\usepackage{listings}

\lstdefinelanguage{Lua} 
{ 
        morekeywords={and,break,do,else,elseif,end,false,for,function,if,in,local,nil,not,or,repeat,return,then,true,until,while,_G,_ENV}, 
        sensitive=true, 
        morecomment=[l]{--}, 
        morecomment=[s]{--[[}{]]}, 
        morestring=[b]", 
        morestring=[b]' 
}

\usepackage{amsthm}
\newtheorem{lemma}{Lemma}

\newcommand\blfootnote[1]{%
  \begingroup
  \renewcommand\thefootnote{}\footnote{#1}%
  \addtocounter{footnote}{-1}%
  \endgroup
}

\usepackage[pagebackref=true,breaklinks=true,letterpaper=true,colorlinks,bookmarks=false]{hyperref}

\cvprfinalcopy 

\def\cvprPaperID{712} 

\ifcvprfinal\fi
\begin{document}

\title{\textsc{TI-pooling}: transformation-invariant pooling \\
for feature learning in Convolutional Neural Networks}

\author{
Dmitry Laptev$^*$ \qquad Nikolay Savinov$^*$ \qquad Joachim M. Buhmann \qquad Marc Pollefeys\\
Department of Computer Science, ETH Zurich, Switzerland\\
{\tt\small \{dlaptev, nikolay.savinov, jbuhmann, marc.pollefeys\}@inf.ethz.ch}
}

\maketitle
\thispagestyle{empty}

\begin{abstract}
\blfootnote{$^*$The authors assert equal contribution and joint first authorship.}
In this paper we present a deep neural network topology that incorporates a simple to implement transformation-invariant pooling operator (\textsc{TI-pooling}). This operator is able to efficiently handle prior knowledge on nuisance variations in the data, such as rotation or scale changes. Most current methods usually make use of dataset augmentation to address this issue, but this requires larger number of model parameters and more training data, and results in significantly increased training time and larger chance of under- or overfitting. The main reason for these drawbacks is that that the learned model needs to capture adequate features for all the possible transformations of the input. On the other hand, we formulate features in convolutional neural networks to be transformation-invariant. We achieve that using parallel siamese architectures for the considered transformation set and applying the \textsc{TI-pooling} operator on their outputs before the fully-connected layers. We show that this topology internally finds the most optimal "canonical" instance of the input image for training and therefore limits the redundancy in learned features. This more efficient use of training data results in better performance on popular benchmark datasets with smaller number of parameters when comparing to standard convolutional neural networks with dataset augmentation and to other baselines.
\end{abstract}

\section{Introduction}
\label{sec:intro}
Recent advances in deep learning lead to impressive results in various applications of machine learning and computer vision to different fields. These advances are largely attributed to expressiveness of deep neural networks with many parameters, that are effectively able to approximate any decision function in the data space \cite{anyfunc}.

While this is true for all the neural network architectures with many layers and with sufficient number of parameters, the most impressive results are being achieved in the fields where deep architectures heavily rely on internal structure of the input data, such as speech recognition, natural language processing and image recognition  \cite{nnstructure}. For example, convolutional neural networks \cite{alexnet} learn kernels to be applied on images or signals reflecting the spatial or temporal dependencies between the neighbouring pixels or moments in time. This structural information serves for internal regularization through weight sharing in convolutional layers \cite{cnnreg}. When combined with the expressiveness of multilayer neural networks, it allows for learning very rich feature representation of input data with little to no preprocessing.

Incorporating structural information permits to work with the inner dependencies in the representation of the data, but only few works have addressed the possible use of other structural prior information known about the data. For example, many datasets in computer vision contain some nuisance variations, such as rotations, shifts, scale changes, illumination variations, etc. These variations are in many cases known in advance from experts collecting the data and one can significantly improve the performance when being considered during training.

The effect is even more explicit when dealing with domain-specific problems. \Eg in many medical imaging datasets, the rotation can be irrelevant due to symmetric nature of some biological structures. At the same time, the scale is fixed during imaging process and should not be considered as a nuisance factor. Moreover scale-invariance can even harm the performance if object size is at least somehow informative, for example, in case of classifying healthy cells from cancer cells \cite{fuchs}. We describe one biomedical example in details in section \ref{sec:experiments:neuro}.

The state of the art approach to deal with these variations and the most popular one in deep learning is \textit{data augmentation} \cite{augmentation} -- a powerful technique that transforms the data point according to some predefined rules and uses it as a separate training sample during the learning procedure. The most common transformations being used in general computer vision are rotations, scale changes and random crops. This approach works especially good when applied with deep learning algorithms, because the models in deep learning are extremely flexible and are able to learn the representation for the original sample and for the transformed ones and therefore are able to generalize also to the variations of the unseen data points \cite{augmentation}. This approach, however, has some limitations listed below.
\begin{itemize}
  \item The algorithm still needs to learn feature representations separately for different variations of the original data. \Eg if a neural network learns edge-detecting features \cite{ciresan} under rotation-invariance setting, it still needs to learn separately vertical and horizontal edge detectors as separate paths of neuron activations.
  \item Some transformations of the data can actually result in the algorithm learning from noise samples or wrong labels. \Eg random crops applied to the input image can capture only a non-representative part of the object in the image, or can fully cut the object out, in which case the algorithm can either overfit to the surrounding or learn from a completely useless representation.
  \item The more variations are considered in the data, the more flexible the model needs to be to capture all the variations in the data. This results in more data required, longer training times, less control over the model complexity and larger potential for overfitting.
\end{itemize}

On the other hand we use the approach inspired by max-pooling operator \cite{maxpooling} and by multiple-instance learning \cite{mil_net} to formulate convolutional neural network features to be transformation-invariant. We take the path of neuron activations in the network and feed it, in a similar manner to augmentation, with the original image and its transformed versions (input instances). But instead of treating all the instances as independent samples, we accumulate all of the responses and take the maximum of them (\textsc{TI-pooling} operator). Because of the maximum, the response is independent from the variations and results in transformation-invariant features that are further propagated through the network. At the same time this allows for more efficient data usage as it learns from only one instance, that already gives maximum response. We call these instances "canonical" and describe in more details in \ref{sec:method:theory}.

This topology is implemented as parallel siamese network \cite{siamese} layers with shared weights and with inputs corresponding to different transformations, described in details in section \ref{sec:method} and sketched in figure \ref{fig:pipeline}. We provide theoretical justification on why features learned in this way are transformation-invariant and elaborate on further properties of \textsc{TI-pooling} in section \ref{sec:method:theory}.

Using \textsc{TI-pooling} permits to learn smaller number of network parameters than when using data augmentation, and lacks a drawback of some data-points missing relevant information after the applied transformation: it only uses the most representative instance for learning and omits the augmentations that are not useful. We review other approaches dealing with nuisance data variations in section \ref{sec:related}.

We evaluate our approach and demonstrate it's properties on three different datasets. The first two are variations of the original MNIST dataset \cite{mnist}, where we significantly outperform the state of the art approaches (for the first variation) or match the current state of the art performance with significantly faster training (on the second variation). The third dataset is a real-world biomedical segmentation dataset with explicit rotation-invariance. On this benchmark we show that incorporating \textsc{TI-pooling} operator increases the performance over the baselines with similar number of parameters, and also demonstrate the property of \textsc{TI-pooling} to find canonical transformations of the input for more efficient data usage.

\section{Related works}
\label{sec:related}

\subsection{Transformation invariant features}

\textbf{Predefined features.} One of the easiest ways to ensure transformation-invariance in most computer vision algorithms is to use specially designed features. The most famous examples of general-purpose transformation-invariant features are SIFT (scale-invariant feature transform) \cite{sift} and its rotation-invariant modification RIFT (rotation-invariant feature transform) \cite{rift}. Another example of domain-specific features is rotation-invariant Line filter transforms \cite{lft}, designed specifically to identify elongated structures for blood vessel segmentation.

Using these features permits the machine learning algorithm to deal with only the inputs that are already invariant to some transformations. But designing the features manually is time-consuming and expensive while not always possible. Furthermore, this approach has two other major limitations: these features are usually not adaptive to the task being solved and they are able to handle only very specific variations in the original data.

\textbf{Feature learning.} To overcome the limitation of task-adaptability, one could use features learned from the input data. \Eg "bag of visual words" \cite{bovw} does not distinguish the positions in which the "visual word" occurs, and therefore it is shift-invariant. With minor modifications, also rotation-invariance can be achieved \cite{rbovw}.

Another approach is transformation-invariant decision jungles (TICJ) \cite{ticdj}. This algorithm learns features to be invariant to any type of predefined transformations as pseudo-linear convolutional kernels and then combines them using a decision tree-like algorithm. Two major limitations of this approach are (i) greediness in the feature learning process (only one kernel is learned at a time) and (ii) relatively low expressiveness of the combining machine-learning algorithm. The algorithms that are usually able to overcome both of these limitations are neural networks.

\subsection{Deep neural networks}
Convolutional deep neural networks \cite{alexnet} are known to learn very expressive features in an adaptive manner depending on the task. Moreover in many cases they resemble some transformation-invariant properties, such as small shift-invariance due to max-pooling layers \cite{maxpooling}. Because the maximum is taken over the neighbouring pixels, local one-pixel shifts usually do not change the output of the subsampling layer. A more general pooling operations \cite{hmax} permit to also consider invariance to local changes that however do not correspond to specific prior knowledge.

To incorporate global transformation-invariance with arbitrary set of transformations, usually the data augmentation is used, as discussed in section \ref{sec:intro}. But also other approaches exist, such as multi-column deep neural networks \cite{multicol} and spatial transformer networks \cite{stn}.

The idea behind multi-column networks is to train different models with the same topology but using different datasets: the original dataset, and the transformed datasets (one separate model is trained for every transformation considered). Then an average of the outputs of individual models is taken to form the final solution.

Spatial transformer networks (STN) follow a completely different idea of looking for a canonical appearance of the input data point. They introduce a new layer to the topology of the network, that transforms the input according to the rules of parametrized class of transformations. The key feature of this approach is that it learns the transformation parameters from the data itself without any additional supervision, except of a defined class of transformations.

The \textsc{TI-pooling} approach in many ways has very similar properties to STN. As we demonstrate in section \ref{sec:method:theory}, our method also finds a canonical position of the input image. But instead of defining a \textit{class} of transformations, we define a more strict \textit{set} of transformations to be considered. In section \ref{sec:experiments:mnist} we show that we achieve similar to STN results on a benchmark introduced by its authors \cite{stn}, but with simpler model and with shorter training time.

\subsection{Multiple instance learning}
Multi-column networks with model averaging described above fall into a category of more general techniques called "multiple instance learning" (MIL) \cite{mil_net}. The area of applications of MIL is very broad, and it can also be applied to train the algorithms invariant to some variations defined as a set of transformations $\Phi$.

Assume that we are given an algorithm $\mathcal{A}$ with some input $x$ and scalar (for simplicity) output $\mathcal{A}(x)$. Then multiple-instance learning approach suggests that algorithm $\mathcal{B}(x)$ will be in many cases transformation-invariant if defined as 
\[
  \mathcal{B}(x) = \max_{\phi \in \Phi} \mathcal{A}(\phi(x))
\]
Instead of a maximum, many different operators can be used (such as averaging), but maximum proves to work best in most applications, so we also focus on it in this work.

The main difference between our approach and MIL is that we propose to learn individual features to be transformation-invariant, and not the algorithm as a whole. Each of the features can then be learned in a way that is most optimal specifically for this feature, allowing different features to rely on different canonical instances and make the most of feature inter-dependencies. We describe this relation in more details in section \ref{sec:method:theory}. This results in our method significantly outperforming the standard MIL models as we show further in section \ref{sec:experiments:neuro}.

Other approaches that are based on the ideas similar to the one presented in this paper are rolling feature maps \footnote{http://benanne.github.io/2015/03/17/plankton.html} and multi-view networks \cite{su2015multi}. The former explores a pooling over a set of transformations, but does not guarantee the transformation-invariance of the features learned. And the latter solves a problem of view invariance, not invariance to an expert-defined set of transformations.

\section{Method description}
\label{sec:method}

\begin{figure*}[t]
\begin{center}
  \includegraphics[width=0.98\textwidth]{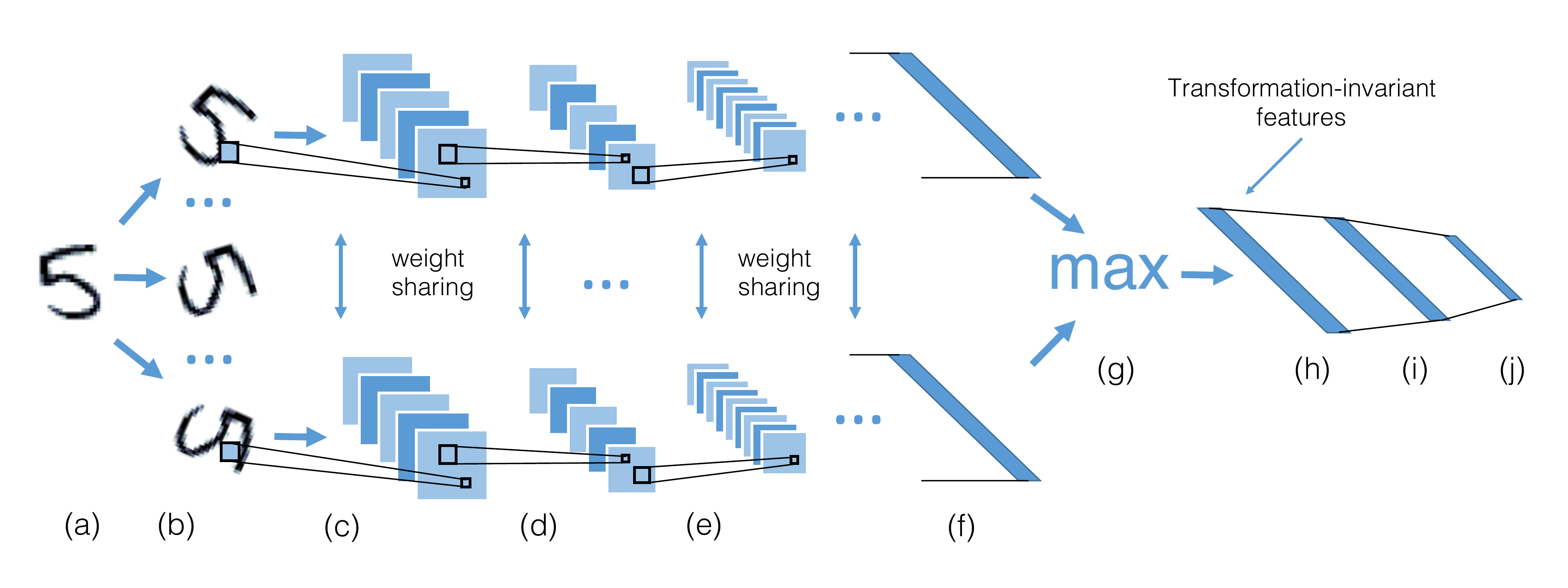}
\end{center}
\caption{Network topology and pipeline description. First, input image $x$ (a) is transformed according to the considered set of transformations $\Phi$ to obtain a set of new image instances $\phi(x), \phi \in \Phi$ (b). For every transformed image, a parallel instance of partial siamese network is initialized, consisting only of convolutional and subsampling layers (two copies are shown in the top and in the bottom of the figure). Every instance is then passed through a sequence of convolutional (c, e) and subsampling layers (d), until the vector of scalars is not achieved (e). This vector of scalars is composed of image features $f_k(\phi(x))$ learned by the network. Then \textsc{TI-pooling} (element-wise maximum) (g) is applied on the feature vectors to obtain a vector of transformation-invariant features $g_k(x)$ (h). This vector then serves as an input to a fully-connected layer (i), possibly with dropout, and further propagates to the network output (j). Because of the weight-sharing between parallel siamese layers, the actual model requires the same amount of memory as just one convolutional neural network. \textsc{TI-pooling} ensures that the actual training of each features parameters is performed on the most representative instance $\phi(x)$.}
\label{fig:pipeline}
\end{figure*}

\subsection{Convolutional neural networks notation}
Convolutional neural networks are usually represented as a sequence of convolutional and subsampling layers with one or more fully-connected layers before the outputs. In this section we for simplicity assume that the input image is two-dimensional (\ie incorporate no colour channels), but the approach generalizes also for colored images. We also omit the explicit notation for activation functions, assuming activations to be incorporated in the specific form of an operator $O$ defined below.

Assume that each neuron performs an operation on the input $x$, that we will refer to as an operator $O(x,\theta)$. It can be either a convolution operator, in which case $\theta$ is a vectorized representation of a convolutional kernel. Or it can be a subsampling operator, which is usually non-parametric, and has no parameters $\theta$. The size of the output matrix $O(x, \theta)$ in each dimension is smaller than the size of $x$ by the size of the kernel in case of a convolution operator, or two times smaller than the input $x$ in case of a subsampling operator.

We refer to these operators applied in layer $l \in \{1, \dots, L\}$ using superscript $l$ on the parameters $\theta$ and we refer to a specific index of the operator within a layer as a subscript. \Eg convolutional operations applied in the first layer of the network can be referred as $O(x, \theta^1_1), \dots, O(x, \theta^1_{n_1})$, where $n_1$ is the number of neurons in layer 1 (we define all the constants in table \ref{tab:topology}). Input to the neuron $i$ in layer $l$ is constructed as a sum of outputs of a previous layer: $\sum_{j=1}^{n_{l-1}} O(x, \theta^{l-1}_{j})$.

We refer to feature $f_k$ of the input image $x$ as an output of a neuron $k$ in a layer that contains only scalar values, \ie layer $l$ such that $O(\cdot, \theta^l_i) \in \mathbb{R}^{1 \times 1}$. Formally this feature is defined as the following composition that we for notation clarity split into a sequence of formulas:
\begin{multline}
\label{eq:f}
  f_i(x) = \sum_{j=1}^{n_{l-1}} O(\cdot, \theta^{l-1}_{j}), \text{where} \\
  O(\cdot, \theta^{l-1}_{k}) = \sum_{j=1}^{n_{l-2}} O(\cdot, \theta^{l-2}_{j}), \text{where} \\
  \ldots \\
  O(\cdot, \theta^{1}_{k}) = O(x, \theta^{1}_{j})
\end{multline}

On top of these features $f_k(x)$, fully-connected layers are usually stacked with some intermediate activation functions, and possibly with dropout masks \cite{dropout} during learning. These details are not directly relevant for this paper and therefore not described in details.

\subsection{Network topology}
Features $f_k(x)$, introduced before, are very powerful when all the parameters $\theta$ are properly trained. They, however, lack a very important property of incorporating any prior information, such as invariance to some known nuisance variations in the data. We fix this property with a relatively easy trick, inspired by multiple-instance learning (MIL) and max-pooling operator.

Assume that, given a set of possible transformations $\Phi$, we want to construct new features $g_k(x)$ in such a way that their output is independent from the known in advance nuisance variations of the image $x$. We propose to formulate these features in the following manner:
\begin{equation}
\label{eq:g}
  g_k(x) = \max_{\phi \in \Phi} f_k(\phi(x))
\end{equation}

We refer to this max-pooling over transformations as to transformation-invariant pooling or \textsc{TI-pooling}. Because of the maximum being applied, every learned feature becomes less dependent on the variations being considered. Moreover, for some sets $\Phi$ we achieve full transformation-invariance, as we theoretically show in section \ref{sec:method:theory}.

As mentioned before and as we show in section \ref{sec:method:theory}, \textsc{TI-pooling} ensures that we use the most optimal instance $\phi(x)$ for learning, and comparing to MIL models we allow every feature $k$ to find its own optimal transformation $\phi$ of the input $x$: $\phi = \arg\max_{\phi \in \Phi}f_k(\phi(x))$.

The topology of the proposed model is also briefly sketched and described in figure \ref{fig:pipeline}.

\textbf{Back-propagation.} Let $\nabla f_k(x)$ be the gradient of the feature $f_k(x)$ defined in equations \ref{eq:f} with respect to the outputs $O(\cdot, \theta^{l-1}_j)$ of the previous layer. This gradient is standard for convolutional neural networks and we do not discuss in details how to compute it. From this gradient we can easily formulate the gradient $\frac{d g_k(x)}{d f_k(x)}$ of the transformation-invariant feature $g_k(x)$ in the following manner:
\[
  \frac{d g_k(x)}{d f_k(x)} = \nabla f_k(\phi(x)), \text{where}~\phi = \arg\max_{\phi \in \Phi} f_k(\phi(x))
\]

The gradient of the neurons of the following fully-connected layer with respect to the output of $g_k(x)$ stays exactly the same as for conventional network topology. Therefore, we have all the building blocks for a back-propagation parameter optimization \cite{backprop}, which concludes the description of \textsc{TI-pooling} and of the proposed topology.

\subsection{Theory and properties}
\label{sec:method:theory}
\textbf{Theoretical transformation-invariance.} Lemma \ref{lemma1}, adapted for our feature from \cite{ticdj}, formulates the conditions on the set $\Phi$ for which the features formulated in equation \ref{eq:g} are indeed transformation-invariant, \ie give exactly the same output for both the original image $x$ and every considered transformation $\phi(x), \phi \in \Phi$.

\begin{lemma}
\label{lemma1}
Let the function $g_k(\cdot)$ to be defined as a maximum over transformations $\phi \in \Phi$ of the function $f_k(\cdot)$. This function is transformation-invariant if the set $\Phi$ of all possible transformations forms a group, \ie satisfies the axioms of closure, associativity, invertibility and identity.
\end{lemma}
\begin{proof}
For the detailed proof please refer to \cite{ticdj}.
\end{proof}

The statement of the lemma is satisfied for many computer vision tasks: simple transformations, such as rotations or non-linear distortions, as well as their compositions form a group. One common example that does not satisfy this property is local shifts (jittering). \Eg if one wants to consider only one pixel shifts, then the closure axiom of the group does not hold: one pixel shift applied twice gives two-pixel shift, which is not in a transformation set.

\begin{figure}[t]
\begin{center}
  \includegraphics[width=0.90\linewidth]{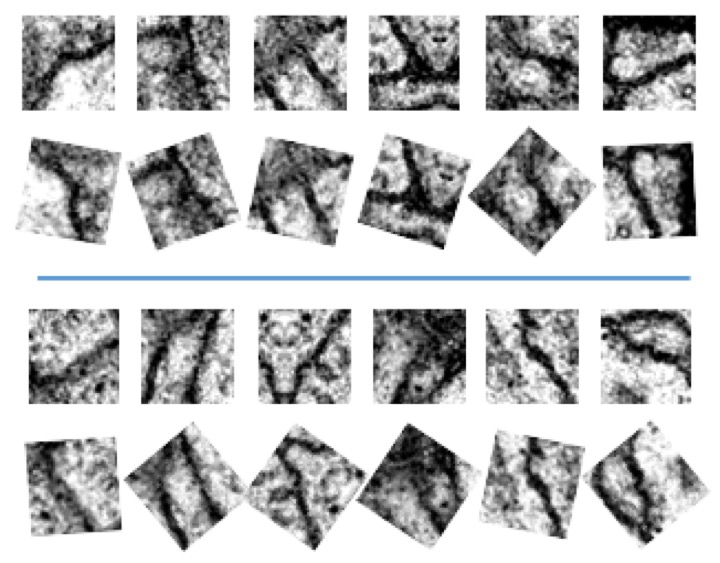}
\end{center}
\caption{First and third rows show the input patches from neuronal segmentation dataset. For this dataset we consider $\Phi$ to be a set of rotations. We then apply a learned model to these patches $x$ and record the angle at which the maximum is achieved for specific feature $g_k(x)$. Then we show the same patches rotated by this angle as shown in rows two and four. One could notice that in most cases the membranes (slightly darker elongated structures) are oriented in approximately the same direction. This means that the algorithm considers this orientation to be canonical for this specific feature and rotates new images to appear similarly.}
\label{fig:canonical}
\end{figure}

\textbf{Canonical position identification.} From a practical point of view, however, the algorithm achieves approximate transformation-invariance even for local transformations. Even if the set $\Phi$ does not form a group, we often observe that the algorithm tries to find a canonical appearance of the image, and then maps a new transformed image to one of the canonical modes. This allows to preserve transformation-invariance in most practical cases with no limitations on $\Phi$. Figure \ref{fig:canonical} shows some examples of neuronal structures oriented in the same manner to a canonical orientation for one of the features.

This property is useful for most problems as it allows for better use of input images. For example, learning discriminative features for every orientation of the image is of course possible with large and deep enough neural network. But assume that now features need to be learned only for canonical orientation (\eg for membranes oriented in all the same direction).

First, for this, much simpler problem, smaller models can be used. Second, the algorithm sees many more examples of canonical vertical edges and therefore can better generalize from them. This brings the next important property of the algorithm.

\begin{figure}[t]
\begin{center}
  \includegraphics[width=0.95\linewidth]{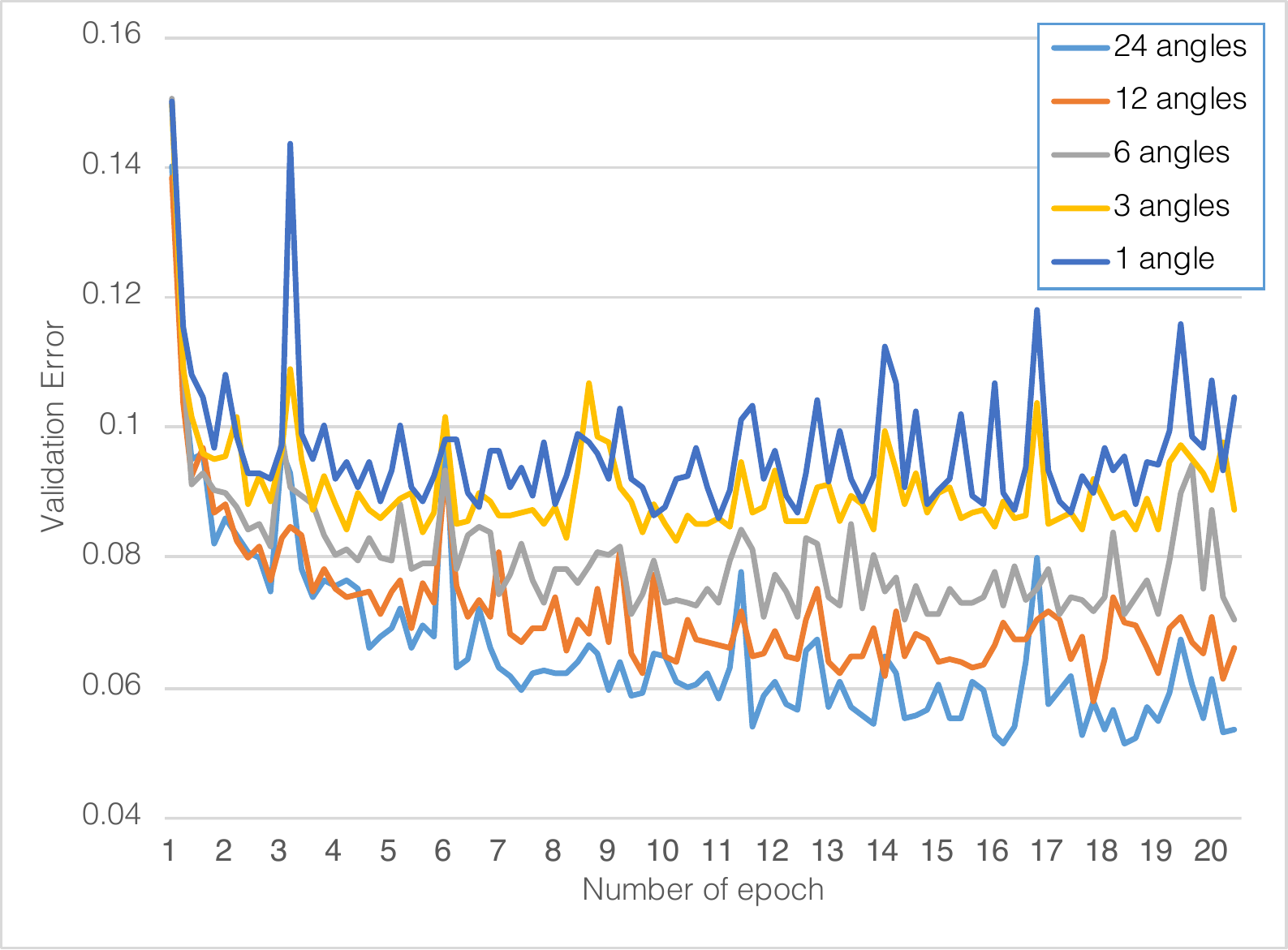}
\end{center}
\caption{Validation error plot for the neuronal segmentation dataset. Depending on how many angles we sample to form a transformation set $\Phi$ (from one, which is equivalent to data augmentation, up to 24) -- the results improve significantly.}
\label{fig:error}
\end{figure}

\textbf{Improved performance and convergence.} Because of more representative examples being used for network training, we observe better performance and convergence rate, when compared with simple data augmentation. Figure \ref{fig:error} shows that the larger the transformation set $\Phi$ -- the better usually the results achieved. This is most probably due to the fact that fewer canonical positions needs to be handled by the learning algorithm.

What \textsc{TI-pooling} is doing to achieve that can be formulated as an exhaustive search over the transformed instances for an instance better corresponding to the current response of the feature. Then only this instance is used to even better improve the performance of the feature. On the other hand, we do not limit all the features to use the same canonical appearance, allowing features to better explore inter-dependencies between the outputs of network layers. We elaborate more on the results in section \ref{sec:experiments}.

\textbf{Any type of transformations.} Another property of the technique, that is worth mentioning, is that it can work with a set of almost any arbitrary transformations. Many works, such as spatial transformer networks \cite{stn}, focus on only limited class of transformations. Those classes can be very rich, \eg include all the possible affine transformations or projections. But still, they need to be differentiable with respect to some defined parameters of the transformation, and, depending on the problem at hand, this can be not enough. \textsc{TI-pooling}, on the contrary, does not rely on differentiability or on any properties of bijective functions or even on the parametrization itself. Examples of common transformations that can be used with our method, and not with \cite{stn} are reflections, most morphological operations and non-linear distortions.

\subsection{Implementation details}
We use Torch7 framework for model formulation and training \cite{torch7}. The easiest way to formulate a proposed model is to use parallel network notation with shared weights as described in figure \ref{fig:pipeline}. The whole model definition requires just few additional lines of code. An example in pseudo-lua code is provided below. Here \texttt{nPhi} is a size of the set $\Phi$.

\lstset{language=Lua}
\begin{lstlisting}
-- define first siamese layers
siamese = Sequential()
...
-- clone and share weights
parallel = Parallel(1, 3)
for phi = 1, nPhi do
  clone = siamese:clone()
  parallel:add(clone)
  clone:share(siamese, 'weight','bias',
              'gradWeight','gradBias')
end
-- formulate a full model
model = Sequential()
model:add(parallel)
model:add(SpatialMaxPooling(nPhi,1,1))
-- add fully-connected layers,
-- dropout and output layer
\end{lstlisting}

The only other modification is to the data: we increase the dimension of the input data tensor by one and stack input instances $\phi(x), \phi \in \Phi$ across the new dimension.

\textbf{Computational complexity.} It may seem like an exhaustive search in the space of possible transformations $\Phi$ significantly increases computational complexity of the pipeline. Indeed, instead of processing one image at a time, we forward-pass $|\Phi|$ images through almost the whole network. We can speed it up by sampling from the space of transformations, but in practice, even searching the full space appears to be more efficient than just data augmentation, because of the following reasons.
\begin{itemize}
  \item Only partial forward pass is done multiple times for the same image, forward-pass through fully-connected layers and back-propagation are exactly the same computationally as for a standard convolutional neural network with the same number of parameters.
  \item Comparing to the data augmentation approach, we make use of every image and it's augmented versions in one pass. Standard convolutional neural network instead makes one pass for every augmented sample, which in the end results in the number of passes equal to the number of augmentated samples to process one image. And because of the previous point, we actually do it more than two times faster.
  \item Because we make use of canonical appearance of the image, the proposed pipeline actually trains more efficiently than the standard neural network, and usually requires smaller number of overall parameters.
\end{itemize}

\section{Experiments}
\label{sec:experiments}
In this section we present the experimental results on three computer vision datasets. The first two datasets are different variations of MNIST dataset \cite{mnist} designed to test artificially-introduced variations in the data. The third one is a neuronal structures segmentation dataset \cite{neuro}, that demonstrates a real-world example of rotation invariance.

\subsection{Rotated MNIST}
\label{sec:experiments:mnist}
Original MNIST dataset \cite{mnist} is a very typical toy dataset to check the performance of new computer vision algorithms modifications. Two variations of MNIST exist to test the performance of different algorithms that are designed to be invariant to some specific variations, such as rotations.

For both the datasets we use the same topology, but slightly different sets $\Phi$. The topology is described in table \ref{tab:topology}. We perform the training using tuning-free convergent adadelta algorithm \cite{adadelta} with the batch size equal to 128 and dropout \cite{dropout} for fully-connected layers.

\begin{table}
\begin{center}
\begin{tabular}{|l|l|}
  \hline
  Layer & Parameters \& channel size \\
  \hline\hline
    input & size: 32x32 \\
    convolution & kernel: 3x3, channel: 40 \\
    relu & ~ \\
    max pooling & kernel: 2x2, stride: 2 \\
    convolution & kernel: 3x3, channel: 80 \\
    relu & ~ \\
    max pooling & kernel: 2x2, stride: 2 \\
    convolution & kernel: 3x3, channel: 160 \\
    relu & ~ \\
    max pooling & kernel: 2x2, stride: 2 \\
    linear & channel: 5120 \\
    relu & ~ \\
    \textsc{TI-pooling} & transformations: $\Phi$ \\
    dropout & rate: 0.5 \\
    linear & channel: 10 \\
    softmax & ~ \\
  \hline
\end{tabular}
\end{center}
\caption{The topology of the network in the experiments.}
\label{tab:topology}
\end{table}

\subsubsection{mnist-rot-12k dataset}
The most commonly used variation of MNIST that is used for validating rotation-invariant algorithms is mnist-rot \cite{mnistrot}. It consists of images from the original MNIST, rotated by a random angle from 0 to $2\pi$ (full circle). This dataset contains 12000 training images, which is significantly smaller, than in the original dataset, and 50000 test samples.

For this dataset we include a \textsc{TI pooling} step over $\Phi$ containing 24 rotations sampled uniformly from 0 to $2\pi$.

We train this network on a single GPU for 1200 epochs and compare the achieved test error with the best results published for this dataset. The best approach by \cite{exp1_3} employs restricted boltzmann machines and achieves 4.2\% error, while we achieve 1.2\% -- the results more than three times better in terms of classification error. The final errors for the proposed and the state of the art results are present in the table \ref{tab:mnistrot}. It can be seen that using \textsc{TI-pooling} indeed leads to significant improvements with no significant effort of optimising topology and just by better exploiting the variations in the data.

\begin{table}
\begin{center}
\begin{tabular}{|l|c|}
  \hline
  Method & Error, \% \\
  \hline\hline
  ScatNet-2 \cite{exp1_1} & 7.48 \\
  PCANet-2 \cite{exp1_2} & 7.37 \\
  TIRBM \cite{exp1_3} & 4.2 \\
  \textsc{TI-pooling} (ours) & \textbf{1.2}  \\
  \hline
\end{tabular}
\end{center}
\caption{Results on mnist-rot-12k dataset.}
\label{tab:mnistrot}
\end{table}

\subsubsection{Half-rotated MNIST dataset}
The second dataset we consider is the dataset introduced in \cite{stn}. There are two reasons why the authors decided to advance further from the original mnist-rot-12k. First, mnist-rot-12k is very small in size (five times less than training set in MNIST dataset). And second, it has somewhat artificial limitation of images being rotated full circle. So they proposed to take full MNIST dataset, use random angle in the range $[-\frac{\pi}{2}, \frac{\pi}{2}]$ (\textit{half} the circle) and use the input images rotated by this angle as training samples. This makes the problem a little easier, but closer to real-world scenarios.

As discussed in section \ref{sec:related}, the authors of spatial transformer networks \cite{stn} propose an elegant way of optimising the transformation of the image while learning also the canonical orientation. Here we show that for some classes of transformations, we achieve comparable results with simpler model and shorter training time.

For this problem formulate a set of transformations $\Phi$ as a set of angles sampled uniformly from half a circle, to match the dataset formulation, overall 13 angles. With this relatively simple model, we converge to the results of 0.8\% error within 360 epochs, while STN was trained for 1280 epochs. Moreover, using \textsc{TI pooling} does not require grid sampling and therefore each individual iteration is faster. With this we still match the performance of the most general STN model defined for a class of projection transformations. For more narrow class of transformations selected manually (affine transformations), our results are slightly worse (by 0.1\%). However, we did not optimise with respect to the transformation classes, and therefore the comparison is not fully fair in this case. Table \ref{tab:halfrotmnist} shows further comparison with STN and other related baselines on this dataset. Baseline fully-connected (FCN) and standard convolutional (CNN) neural networks are defined in \cite{stn} and tuned to have approximately the same number of parameters as the baseline STN.

\begin{table}
\begin{center}
\begin{tabular}{|l|c|}
  \hline
  Method & Error, \% \\
  \hline\hline
  FCN & 2.1 \\
  CNN & 1.2 \\
  STN (general) & 0.8 \\
  STN (affine) & 0.7 \\
  \textsc{TI-pooling} (ours) & 0.8  \\
  \hline
\end{tabular}
\end{center}
\caption{Results on half-rotated MNIST dataset.}
\label{tab:halfrotmnist}
\end{table}

\subsection{Neuronal structures segmentation}
\label{sec:experiments:neuro}
The neuronal membrane segmentation dataset was used for ISBI 2012 challenge \cite{neuro}. It consists of images of neuronal tissue captured with serial section transmission electron microscopy. The task is to perform pixel segmentation into two classes: cell membranes and inner parts of the neuron. We take 2x downsampled images and split them into training and test sets: first 25 sections are used for training, and the last five are left for test.

From the neurological experts we know, that membrane appearance does not depend on the orientation of the membrane, and therefore we can safely include $[0, 2\pi]$ rotations in the set of transformations $\Phi$. We sample rotations every 15 degrees, resulting in 24 transformations considered.

Because this is a segmentation task, we extract patches around a pixel and classify those patches (here label of the patch is the label of the central pixel of the patch). We perform training on all the available pixel patches (balanced between classes). The patch is decided to be square and has the size of $46$ pixel, but after the rotation we crop the patch, so the actual input to the network is a $32 \times 32$ patch. Some examples of the patches are present in figure \ref{fig:canonical}.

For every algorithm we run for this dataset, we select the same network topology, in order to better evaluate the improvement of the proposed \textsc{TI-pooling} operator for rotation-invariant feature learning without incorporating any other effects such as model size. As our baselines, we select the following two algorithms, that are closely related to the proposed technique as discussed in sections \ref{sec:intro} and \ref{sec:related}:
\begin{itemize}
  \item standard convolutional neural network with data augmentation, that is able ideally to learn features expressive enough to handle rotations in the data;
  \item multiple instance learning of convolutional neural networks, that is able to learn a transformation-invariant algorithm for a given set of transformations, but not the features.
\end{itemize}

\begin{table}
\begin{center}
\begin{tabular}{|l|c|}
  \hline
  Method & Error, \% \\
  \hline\hline
  MIL over CNN \cite{mil_net}       & 8.9           \\
  CNN with augmentation \cite{alexnet}  & 8.1           \\
  \textsc{TI-pooling} - dropout     & \textbf{7.4}  \\
  \textsc{TI-pooling} + dropout     & \textbf{7.0}  \\
  \hline
\end{tabular}
\end{center}
\caption{Results on neuronal segmentation dataset.}
\label{tab:neuro}
\end{table}

For all the underlying networks we select the same topology as described in table \ref{tab:topology}, except of \textsc{TI-pooling} and the number of outputs (two classes for this dataset). We also report the results with and without dropout, as discussed later.

Table \ref{tab:neuro} shows the pixel error achieved by all the algorithms after 16 epochs. To make the comparison absolutely fair, for standard convolutional neural network with augmentation we record the results after $16*24 = 384$ epochs, so that the number of images "seen" by the algorithm is the same (because for the proposed approach and for the MIL modification, we take the maximum over all the 24 rotations in one iteration). We also run MIL modification with no dropout, and compare the results with the version of our algorithm trained with no dropout. For both baselines we see the significant improvement for the same topology. From this we can conclude that the proposed \textsc{TI-pooling} is indeed very helpful for real-world problems with nuisance variations.

\section{Conclusions}
\label{sec:conclusions}
In this paper we propose a novel framework to incorporate expert knowledge on nuisance variations in the data when training deep neural networks. We formulate a set of transformations that should not affect the algorithm decision and generate multiple instances of the image according to these transformations. Those instances, instead of being used for training independently, are passed through initial layers the network and through \textsc{TI-pooling} operator to form transformation-invariant features. These features are fully-trainable using back-propagation, have the rich expressiveness of standard convolutional neural network features, but at the same time do not depend on the variations in the data.

Convolutional neural network with \textsc{TI-pooling} has some theoretical guarantees of being transformation-invariant algorithm for variations common in computer vision problems. But more importantly, it has some nice practical properties that we show in this work, \eg it permits to learn from the most representative instances, that we call "canonical". Because of that the network does not have to learn features separately for every possible variation of the data from augmented samples, but instead it learns only features that are relevant for one appearance of the image, and then applies it for all the variations. It also allows to better use the input data to learn these features: \eg all the samples including edges participate in learning transformation-invariant edge detector feature, and no separate vertical or horizontal edge detector features are needed.

We test the method on three datasets with explicitly defined variability. In all the experiments we either significantly outperform or at least match the performance of baseline state of the art techniques. Often we also show faster convergence rates than baselines with smaller yet smarter data-aware models.

The proposed \textsc{TI-pooling} operator can be used as a separate neuronal unit for most networks architectures with very little effort to incorporate prior knowledge on nuisance factors in the data. But the range of its applications goes well beyond that, allowing to incorporate many types of prior information on the data and opening the opportunities for more robust expert-driven algorithms in combination with the powerful expressiveness of deep learning.

\textbf{Acknowledgements:}
The work is partially funded by the Swiss National Science Foundation projects 157101 and 163910.

{\small
\bibliographystyle{ieee}
\bibliography{laptev_cvpr16}

\begin{thebibliography}{10}\itemsep=-1pt

\bibitem{maxpooling}
Y.-L. Boureau, J.~Ponce, and Y.~LeCun.
\newblock A theoretical analysis of feature pooling in visual recognition.
\newblock In {\em Proceedings of the 27th International Conference on Machine
  Learning (ICML-10)}, pages 111--118, 2010.

\bibitem{siamese}
J.~Bromley, J.~W. Bentz, L.~Bottou, I.~Guyon, Y.~LeCun, C.~Moore,
  E.~S{\"a}ckinger, and R.~Shah.
\newblock Signature verification using a “siamese” time delay neural
  network.
\newblock {\em International Journal of Pattern Recognition and Artificial
  Intelligence}, 7(04):669--688, 1993.

\bibitem{exp1_1}
J.~Bruna and S.~Mallat.
\newblock Invariant scattering convolution networks.
\newblock {\em Pattern Analysis and Machine Intelligence, IEEE Transactions
  on}, 35(8):1872--1886, 2013.

\bibitem{neuro}
A.~Cardona, S.~Saalfeld, S.~Preibisch, B.~Schmid, A.~Cheng, J.~Pulokas,
  P.~Tomancak, and V.~Hartenstein.
\newblock An integrated micro-and macroarchitectural analysis of the drosophila
  brain by computer-assisted serial section electron microscopy.
\newblock {\em PLoS biology}, 8(10):e1000502, 2010.

\bibitem{exp1_2}
T.-H. Chan and et~al.
\newblock Pcanet: A simple deep learning baseline for image classification?
\newblock 24, 2015.

\bibitem{backprop}
Y.~Chauvin and D.~E. Rumelhart.
\newblock {\em Backpropagation: theory, architectures, and applications}.
\newblock Psychology Press, 1995.

\bibitem{ciresan}
D.~Ciresan, A.~Giusti, L.~M. Gambardella, and J.~Schmidhuber.
\newblock Deep neural networks segment neuronal membranes in electron
  microscopy images.
\newblock In {\em Advances in neural information processing systems}, pages
  2843--2851, 2012.

\bibitem{multicol}
D.~Ciresan, U.~Meier, and J.~Schmidhuber.
\newblock Multi-column deep neural networks for image classification.
\newblock In {\em Computer Vision and Pattern Recognition (CVPR), 2012 IEEE
  Conference on}, pages 3642--3649. IEEE, 2012.

\bibitem{torch7}
R.~Collobert, K.~Kavukcuoglu, and C.~Farabet.
\newblock Torch7: A matlab-like environment for machine learning.
\newblock In {\em BigLearn, NIPS Workshop}, number EPFL-CONF-192376, 2011.

\bibitem{dropout}
G.~E. Hinton, N.~Srivastava, A.~Krizhevsky, I.~Sutskever, and R.~R.
  Salakhutdinov.
\newblock Improving neural networks by preventing co-adaptation of feature
  detectors.
\newblock {\em arXiv preprint arXiv:1207.0580}, 2012.

\bibitem{stn}
M.~Jaderberg, K.~Simonyan, A.~Zisserman, and K.~Kavukcuoglu.
\newblock Spatial transformer networks.
\newblock 2015.

\bibitem{alexnet}
A.~Krizhevsky, I.~Sutskever, and G.~E. Hinton.
\newblock Imagenet classification with deep convolutional neural networks.
\newblock In {\em Advances in neural information processing systems}, pages
  1097--1105, 2012.

\bibitem{ticdj}
D.~Laptev and J.~M. Buhmann.
\newblock Transformation-invariant convolutional jungles.
\newblock In {\em Proceedings of the IEEE Conference on Computer Vision and
  Pattern Recognition (CVPR 2015)}, 2015.

\bibitem{mnistrot}
H.~Larochelle, D.~Erhan, A.~Courville, J.~Bergstra, and Y.~Bengio.
\newblock An empirical evaluation of deep architectures on problems with many
  factors of variation.
\newblock In {\em Proceedings of the 24th international conference on Machine
  learning}, pages 473--480. ACM, 2007.

\bibitem{rift}
S.~Lazebnik, C.~Schmid, J.~Ponce, et~al.
\newblock Semi-local affine parts for object recognition.
\newblock In {\em British Machine Vision Conference (BMVC'04)}, pages 779--788,
  2004.

\bibitem{nnstructure}
Y.~LeCun and Y.~Bengio.
\newblock Convolutional networks for images, speech, and time series.
\newblock {\em The handbook of brain theory and neural networks}, 3361(10),
  1995.

\bibitem{cnnreg}
Y.~LeCun, L.~Jackel, L.~Bottou, A.~Brunot, C.~Cortes, J.~Denker, H.~Drucker,
  I.~Guyon, U.~Muller, E.~Sackinger, et~al.
\newblock Comparison of learning algorithms for handwritten digit recognition.
\newblock In {\em International conference on artificial neural networks},
  volume~60, pages 53--60, 1995.

\bibitem{anyfunc}
M.~Leshno, V.~Y. Lin, A.~Pinkus, and S.~Schocken.
\newblock Multilayer feedforward networks with a nonpolynomial activation
  function can approximate any function.
\newblock {\em Neural networks}, 6(6):861--867, 1993.

\bibitem{mnist}
C.-L. Liu and et~al.
\newblock Handwritten digit recognition: benchmarking of state-of-the-art
  techniques.
\newblock {\em Pattern Recognition}, 36(10):2271--2285, 2003.

\bibitem{sift}
D.~G. Lowe.
\newblock Object recognition from local scale-invariant features.
\newblock In {\em Computer vision, 1999. The proceedings of the seventh IEEE
  international conference on}, volume~2, pages 1150--1157. Ieee, 1999.

\bibitem{bovw}
J.~Philbin, O.~Chum, M.~Isard, J.~Sivic, and A.~Zisserman.
\newblock Lost in quantization: Improving particular object retrieval in large
  scale image databases.
\newblock In {\em Computer Vision and Pattern Recognition, 2008. CVPR 2008.
  IEEE Conference on}, pages 1--8. IEEE, 2008.

\bibitem{hmax}
M.~Riesenhuber and T.~Poggio.
\newblock Hierarchical models of object recognition in cortex.
\newblock {\em Nature neuroscience}, 2(11):1019--1025, 1999.

\bibitem{lft}
K.~Sandberg and M.~Brega.
\newblock Segmentation of thin structures in electron micrographs using
  orientation fields.
\newblock {\em Journal of structural biology}, 157(2):403--415, 2007.

\bibitem{fuchs}
P.~J. Sch{\"u}ffler, T.~J. Fuchs, C.~S. Ong, V.~Roth, and J.~M. Buhmann.
\newblock Computational tma analysis and cell nucleus classification of renal
  cell carcinoma.
\newblock In {\em Pattern Recognition}, pages 202--211. Springer, 2010.

\bibitem{exp1_3}
K.~Sohn and H.~Lee.
\newblock Learning invariant representations with local transformations.
\newblock In {\em Proceedings of the 29th International Conference on Machine
  Learning (ICML-12)}, pages 1311--1318, 2012.

\bibitem{su2015multi}
H.~Su, S.~Maji, E.~Kalogerakis, and E.~Learned-Miller.
\newblock Multi-view convolutional neural networks for 3d shape recognition.
\newblock In {\em Proceedings of the IEEE International Conference on Computer
  Vision}, pages 945--953, 2015.

\bibitem{augmentation}
D.~A. Van~Dyk and X.-L. Meng.
\newblock The art of data augmentation.
\newblock {\em Journal of Computational and Graphical Statistics}, 10(1), 2001.

\bibitem{mil_net}
J.~Wu, Y.~Yu, C.~Huang, and K.~Yu.
\newblock Deep multiple instance learning for image classification and
  auto-annotation.
\newblock In {\em Proceedings of the IEEE Conference on Computer Vision and
  Pattern Recognition}, pages 3460--3469, 2015.

\bibitem{rbovw}
Y.~Yang and S.~Newsam.
\newblock Bag-of-visual-words and spatial extensions for land-use
  classification.
\newblock In {\em Proceedings of the 18th SIGSPATIAL International Conference
  on Advances in Geographic Information Systems}, pages 270--279. ACM, 2010.

\bibitem{adadelta}
M.~D. Zeiler.
\newblock Adadelta: An adaptive learning rate method.
\newblock {\em arXiv preprint arXiv:1212.5701}, 2012.

\end{thebibliography}
}

\end{document}